%% file: p_17_iros_hyq_subsystem_control.tex
\documentclass[letterpaper, 10 pt, conference]{ieeeconf}  

\IEEEoverridecommandlockouts                              

\usepackage{graphics} 
\usepackage{epsfig} 
\usepackage{mathptmx} 
\usepackage{times} 
\usepackage{amsmath} 
\usepackage{amssymb}  
\usepackage{todonotes}
\usepackage{color}
\usepackage{romannum}
\usepackage{savesym} \savesymbol{AND}
\usepackage[noadjust]{cite}
\usepackage{hyperref}
 
\newtheorem{theorem}{Theorem}

\input{mathdef}

\title{\LARGE \bf
Robust Whole-Body Motion Control of Legged Robots
}

\author{Farbod Farshidian, Edo Jelavi\'c, Alexander W. Winkler, Jonas Buchli  
\thanks{$^*$All authors are with the Agile \& Dexterous Robotics Lab, ETH Z\"urich, Switzerland, email: \{farbodf, jelavice, winklera, buchlij\}@ethz.ch}%
}

\begin{document}

\maketitle
\thispagestyle{empty}
\pagestyle{empty}

\begin{abstract}
We introduce a robust control architecture for the whole-body motion control of torque controlled robots with arms and legs. The method is based on the robust control of contact forces in order to track a planned Center of Mass trajectory. Its appeal lies in the ability to guarantee robust stability and performance despite rigid body model mismatch, actuator dynamics, delays, contact surface stiffness, and unobserved ground profiles. Furthermore, we introduce a task space decomposition approach which removes the coupling effects between contact force controller and the other non-contact controllers. Finally, we verify our control performance on a quadruped robot and compare its performance to a standard inverse dynamics approach on hardware.
\end{abstract}

\section{INTRODUCTION}
The robustness of motion controllers for legged robots has recently attracted a significant research interest. Generally in any robotic application, a generic motion control requires three main components namely, (i) trajectory planning which is responsible for generating reference trajectories, (ii) a state observer to estimate the state of the robot and (iii) a stabilizing feedback controller which is responsible for tracking the reference trajectories and rejecting disturbances. For any successful hardware implementation, the robustness should be considered in designing each of these three modules. Here, we focus on the robustness of the stabilizing controller. 

Task space inverse dynamics is a popular approach for whole--body control of legged robots given that the reference trajectory has already been calculated \cite{righetti2011, winkler17}. Hierarchical inverse dynamics based on cascades of quadratic programs has been proposed for the control of legged robots and implemented on torque controlled legged robots \cite{hutter2013, herzog2016}. Inverse kinematics based approaches with constraints and hierarchical control approach such as \cite{mistry2007task} have also been used for the control of legged robots. All of these methods are in general a versatile tool for implementing already existing reference trajectories. However, the main limitation of these methods is their assumption and reliance on ideal force tracking actuators. Thus, they are often susceptible to the actuators dynamics, model uncertainties, and delays. 

In order to increase the robustness of a motion controller, different schemes have been proposed. A robust inverse dynamics method is proposed by \cite{del2016, nguyen2015} which makes the task constraints less susceptible to the additive uncertainty in joint torques. In \cite{kelly1989, lee2016}, an adaptive control method has been used to estimate or improve the rigid body model in order to adapt the controller. Task space robustness of manipulator's motion has been considered by \cite{becker03, cheah03} where the proposed method guarantees stability in the presence of rigid body model mismatch and actuator dynamics.   
While torque control algorithms with high-performance \cite{boaventura12} can reinforce the perfect actuator assumption to a certain extent, their performance is still not ideal. In this paper, we propose a feedback control structure which can increase the robustness of the controller in face of actuator dynamics, finite contact surface stiffness, and unobserved ground profile. We assume that the objective of the whole-body stabilizing controller is to track planned trajectories of the Center of Mass (CoM). Considering that the CoM motion is only affected by the external forces, a robust implementation of the given plan can be achieved by robustly controlling the robot's interaction ports (e.g. footsteps or contact forces) with the environment.

The main contributions of this work are as follows: First, we introduce an Internal Model Control (IMC) structure to feedback control the end-effector (e.g. stance leg) contact force by considering the actuator dynamics and delays. Then, we demonstrate that we can achieve the robust stability and performance of our IMC contact force controller with the proper tuning of IMC; the tuning boils down to the selection of a single parameter. Furthermore, in order to decouple the contact force controller and the non-contact end-effector controller (e.g. swing leg), we introduce a general approach for either hierarchical or co-equal task space decomposition into subsystems. While this approach is formulated for fully-actuated legged robots extension to under-actuated cases is conceivable. Finally, we discuss the experimental results on our quadrupedal robot, HyQ. and compare it to a standard whole-body control method of legged robot used in \cite{winkler17}.

\section{Background}
In this section, we briefly introduce the task space modeling approach and model-based feedback force control structure which we will use later in our proposed method. 
 
\subsection{Task Space}
Assuming rigid body modelling, we can write the equations of motion of a robot as
\begin{align}
&\vM(\vq) \ddot{\vq} + \vC(\vq,\dot{\vq}) + \vG(\vq) = \vS^\top \vtau + \vJ^\top_{c_\bot} \vlambda_{c_\bot} + \vJ^\top_{c_\parallel} \vlambda_{c_\parallel}  \notag \\
&\vJ_{c_\parallel} \dot{\vq} = 0,  \quad \vJ_{c_\bot} \dot{\vq} = 0,  \quad \vlambda_{\bot} \geq 0
\label{eq:rigidbody_model}
\end{align} 
where $\vM$, $\vC$, $\vG$, and $\vS$ are respectively the generalized inertia matrix, Coriolis-centrifugal forces, generalized gravity forces, and selection matrix. $\vq$ is the generalized coordinate vector which is augmented by the floating base pose coordinate. $\vlambda_{c_\bot}$ and $\vlambda_{c_\parallel}$ are the orthogonal and tangential contact forces. Their corresponding Jacobian are $\vJ_{c_\bot}$ and $\vJ_{c_\parallel}$ respectively and the equality and the inequality equations are the constraints that are introduced through contacts. The task space dynamics can be obtained by mapping the whole-body model in Equation~\eqref{eq:rigidbody_model} to the range space of the task Jacobian matrix, $\vJ_x$,  \cite{khatib87}. Thus, we can write
\begin{equation}  \label{eq:task_space_equation}
\vM_{x}(\vq) \ddot{\vx} + \vC_{x}(\vq,\dot{\vq}) + \vG_{x}(\vq) = \vS_{x}^\top \vtau +  \vJ_{c,x}^\top \vlambda ,
\end{equation}
where we have 
\begin{align}
\label{eq:task_space_matrices}
\begin{aligned}[c] 
& \vM_{x} = \vJ_{x}^{\dagger \top} \vM \vJ_{x}^{\dagger} \\
& \vC_{x} = \vJ_{x}^{\dagger \top} \vC - \vM_{x} \dot{\vJ}_x \dot{\vq}  \\
& \vG_{x} = \vJ_{x}^{\dagger \top} \vG  
 \end{aligned}
 \qquad
 \begin{aligned}[c]
& \vS_{x} = \vS \vJ_{x}^{\dagger} \\
& \vJ_{c,x} = \vJ_c \vJ_{x}^{\dagger} \\
& \vJ_{x}^{\dagger} = \vM^{-1} \vJ_{x}^\top \left( \vJ_{x} \vM^{-1} \vJ_{x}^\top  \right)^{-1} 
\end{aligned}
\end{align}
$\vJ_{x}^{\dagger}$ is the dynamically consistent right pseudo-inverse of the full row rank
Jacobian matrix, $\vJ_x$. In the next section, we use this equation to derive motion dynamics of three task spaces defined in our subsystem decomposition approach.

\subsection{Internal Model Control (IMC)}
The IMC structure is an alternative parametrization over the stable classical feedback controllers for linear systems. However, in contrast to the classical approach which uses the system's measured outputs for feedback control, IMC directly utilizes the system dynamics model for calculating a feedback signal based on the difference between the measured outputs and the predicted outputs of its internal model (Fig.~\ref{fig:imc}). Moreover, it uses an inverse model approach in order to improve its tracking and disturbance rejection. Another major advantage of IMC is that it allows a systematic tradeoff between performance and robustness \cite{morari89} since the closed loop sensitivity function is linear with respect to its controller. In the next section, we explain the design process of a robust IMC structure in more details.    

To the best of our knowledge, IMC has not been applied for task space force control. However, a closely related method known as disturbance observer (DOB) has been recently implemented for force control in low-level control of actuators \cite{haninger2016,roozing2016comparison}. Strictly speaking, IMC and DOB are not the same. However, they both rely on the internal model principle. The only notable exception for DOB application in rigid body level is \cite{vorndamme2016} in which DOB is used to achieve the precise tracking on a humanoid arm while being compliant. 

\section{Approach}
The main objective of this paper is to increase the robustness of the motion controller for tracking given CoM trajectories. Based on the Newton-Euler equation of motion, the CoM dynamics are only affected by external forces on the robot. Thus, a robust implementation of a given CoM plan can be achieved by robustly controlling the contact forces and contact point locations at the moments of touchdown. The robust and compliant control of the free end-effectors (e.g. swing legs and arms) has been extensively studied in the literature \cite{vorndamme2016}. Therefore, we mainly focus on the robust control of the contact forces. First, we introduce a new contact force controller which performs robustly in face of rigid body model mismatch, actuator dynamics, delays, contact surface stiffness, and unobserved ground profiles. Then, to reduce the coupling effects between controllers, we propose a generalized approach for controlling two task spaces with similar precedence. Finally, we explain our control approach for CoM and non-contact degrees of freedom (DoF). 

In our proposed method, the legged robot model is broken down into three subsystems, namely: contact subsystem, CoM subsystem, and non-contact subsystem as shown in Fig.~\ref{fig:control_structure} (the blue block). The contact subsystem includes DoF in system which directly manipulate the contact forces. The CoM subsystem is attributed to the whole-body CoM dynamics. At last, the non-contact subsystem comprises of any remaining DoF such as the swing legs. 

\begin{figure}[tbp]
    \includegraphics[width=\columnwidth]{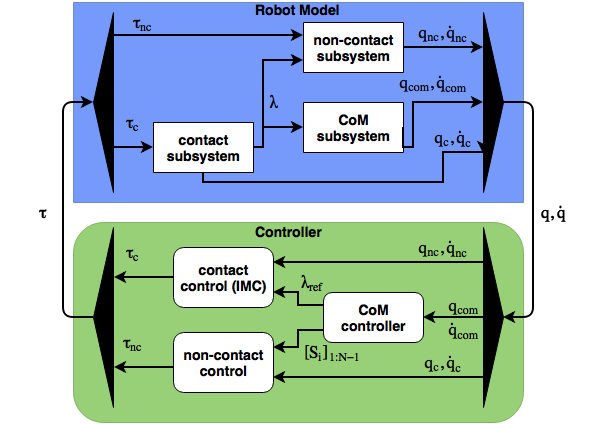}
    \caption{Overview of the model (blue) and control structure (green). The 
    system model is separated into three subsystems, a CoM subsystem, a
    contact subsystem and a non-contact subsystem. Each of these systems
    is controlled with a corresponding controller. The black triangles describe
    the mappings between subsystem states and generalized coordinates $q$ 
    as well as between subsystems torques and joint torques $\tau$.}
    \label{fig:control_structure}
\end{figure}

\subsection{Contact Force Control at End--Effectors}
The contact subsystem is defined in the range space of the end-effectors' Jacobians $\vJ_c$ that are in contact with the environment (e.g. stance legs). The equation of motion for this subsystem has the same form as \eqref{eq:task_space_equation}, thus we have
\begin{equation}  \label{eq:contact_subsystem_1}
\vC_{c}(\vq,\dot{\vq}) + \vG_{c}(\vq) = \vS_{c}^\top \vtau +  \vlambda, \qquad \vlambda_{\bot} \geq 0 ,
\end{equation}
where the end-effector acceleration, $\ddot{\vx}_{c}$, is set to zero due to the equality constraints in Equation~\eqref{eq:task_space_equation} for non-slipping contact points. Notice that the projected contact forces Jacobian in this equation is an identity matrix since $\vJ_{c,c} =  \vJ_c \vJ_{c}^{\dagger} = \vI$. The assumption that $\vJ_c$ is a full row rank matrix holds as long as there exist at least three independent DoF for each contact point. Furthermore, conveniently all of the contact related equality and inequality constraints are assimilated in Equation~\eqref{eq:contact_subsystem_1}. Assuming that all of the joints are actuated, $\vS_{c}^\top$ has the same range space as $\vJ_c$ which renders this subsystem as fully actuated. Therefore, we get
\begin{equation}  \label{eq:contact_subsystem_2}
\vC_{c}(\vq,\dot{\vq}) + \vG_{c}(\vq) = \vtau_c +  \vlambda,  \qquad \vlambda_{\bot} \geq 0
\end{equation}
where $\vtau_c = \vS_{c}^\top \vtau$. Notice that the mapping from $\vtau_c$ to $\vtau$ is not unique since for any given $\vtau_c$, there exists a subspace of dimension $Dim(\textit{Ker}(\vJ_c))$ in the generalized force space that produces the same value. Later, we will use this characteristic to decouple the contact and non-contact subsystems.

In order to control this subsystem, we first feedback-linearize its dynamics by setting $\vtau_c = \vF + \vC_{c}(\vq,\dot{\vq}) + \vG_{c}(\vq)$. Then, we get
\begin{equation}  \label{eq:contact_subsystem_fl}
\vlambda = -\vF, \qquad \vlambda_{\bot} \geq 0.
\end{equation}
According to this equation, the relation between contact forces and $\vF$ is algebraic. However, this equation is based on the rigid body modeling and hard contact assumption and does not consider the actuator dynamics and softness of the contact surfaces ($\ddot{\vx}_{c}\neq 0$). In order to design a controller which can deal with these issues, we propose a feedback structure known as IMC. To guarantee the robustness of this structure, we need to make a few assumptions on the system dynamics and the external signal type (such as references and disturbances). First, we assume that all of the actuators have the same nominal dynamics. This assumption causes the projected force $\vtau_c$ and consequently $\vF$ to have the same nominal dynamics as actuators. We use here a first order system with delay (first order deadtime system). Thus, for each component $i$, in the contact force vector, we have the following dynamics 
\begin{equation} \label{eq:contact_subsystem_linear_model}
\tilde{H}_p(s) = \frac{\lambda_i}{u_i} = - \frac{e^{-\tilde{\eta_d} s}}{\tilde{\eta} s + 1} ,
\end{equation}
where $\tilde{H}_p(s)$ is the nominal model between the $i$th contact force component and the $i$th element of the projected feedback-linearized actuator command $u_i$. $\tilde{\eta_d}$ and $\tilde{\eta}$ are the nominal delay and the time constant. In this nominal model, we assume that the system does not have a steady state error to a constant input (the gain is one) which is the case for ideal actuators. Finally, for modeling uncertainties of this model, we assume that the actual transfer function of system, $H_p(s)$, belong to the family $\Pi$ defined as
\begin{equation} \label{eq:contact_subsystem_unceratinity}
\Pi= \left\{ H_p(s) \mid \left\lVert \frac{H_p(j\omega)-\tilde{H}_p(j\omega)}{\tilde{H}_p(j\omega)} \right\rVert \leq \overline\l(\omega) \right\}
\end{equation}
which means that relative error in each frequency is bounded by $\overline\l(\omega)$. A typical shape of $\overline\l(\omega)$ can be found in Fig.~\ref{fig:robust_perfomrance}. In practice, we can also assume that $\overline\l(0) < 1$ which is equivalent to say that the constant input to the actuators produces a steady state force output with less than $100\%$ relative tracking error. This is a reasonable assumption for many force controlled robots. The external signals such as the reference forces and disturbances are assumed to be step inputs \footnote{Note that any arbitrary signal can be approximated by a finite set of shifted step functions}. This class of signals can capture the sudden changes in the contact force reference; such changes appear when establishing/breaking contacts as well as  disturbances such as slippage of the legs and push forces on the robot's body.

The proposed IMC structure is depicted in Fig.~\ref{fig:imc}. Since there is no interconnection between non-pair inputs and outputs, we use a single-input single-output (SISO) control structure for each input-output pairs. Our SISO IMC consists of two controllers, namely the disturbance rejection controller, $q_d(s)$, and the reference tracking controller, $q_r(s)$. The disturbance rejection controller uses the estimated disturbance signal $\tilde{\vd}$ to reject disturbances, $\vd$ and the reference tracking controller modulates the input signals in order to track the desired output (refer to Fig.~\ref{fig:imc}). 

Both controllers use a notion of the system inverse model ($\tilde{H}^{-1}_p(s)$) for improving the closed loop force control performance. In general, the design process of IMC consists of two steps: (i) design an optimal controller for the nominal plant and (ii) detune the nominal control to obtain robust performance. The first step of the design uses an $H_2$-optimal control method which minimizes the Integral of Square Error (ISE) criterion for the step type input. The latter step often uses a lowpass filter in order to guarantee robust stability and robust performance. The $H_2$-optimal controller has no tracking error since for ISE being finite, the reference tracking error should asymptotically approach zero. The ISE criterion is defined as:
\begin{equation*}
\min \left\lVert e \right\lvert^2_2 = \int_{0}^{\infty}{ \lVert \lambda(t)-\lambda_r(t)\rVert^2_2 dt } 
\end{equation*}
Based on $H_2$-optimal theorem \cite{morari89} the optimal controller is $q(s)= (\eta s + 1) f(s)$
for a first order deadtime system with step input which is obtained based on the nominal model inverse (${H}^{-1}_p(s)$). Here, $f(s)$ is a lowpass filter. For optimal solution $f(s)$ should be unity. However, there are two major issues with this choice: first, since the degree of the nominator is higher than the degree of the denominator, this controller will not be a causal system. Second, this system has a high amplitude in the high frequencies which makes it susceptible to sensory noise and model uncertainty such as non-stiff ground (which produces high frequency $\ddot{\vx}_c$). To overcome these issues we use a lowpass filter with a proper structure. The following theorem states that there exists such a filter which can guarantee stability and zero tracking error for all the system in the family $\Pi$.    

\begin{theorem}
Assume $\overline\l(\omega)$ is continuous and $\overline\l(0) < 1$. Then there exists a lowpass filter $f(s)$ such that the closed loop system is robustly stable for family $\Pi$ and it tracks asymptotically error-free constant inputs.
\end{theorem}

\begin{proof}
The proof is based on the corollary~(4.3-2) and corollary~(4.4-2) in \cite{morari89} and the assumption $\overline\l(0) < 1$.
\end{proof}   

In practice, this filer should be designed by using the minimum order lowpass filter which results in a satisfactory performance. On the system used in this paper, a first order lowpass filter with one free parameter $\eta_f$ was enough for robust stability and performance. 
\begin{equation} \label{eq:lowpass_filter}
f(s) = \frac{1}{\eta_f s + 1}
\end{equation}
This parameter can be tuned directly on the hardware, by gradually increasing $\eta_f$ until a satisfactory performance is obtained. Since $\eta_f$ determines the response time of the controller, increasing $\eta_f$ reduces the performance of the closed loop system in favour of robustness. In order to analytically guarantee the robust performance of the controller for all systems in the family $\Pi$, we need to show that the following condition holds for all the frequencies
\begin{equation}  \label{eq:robust_performance}
\left\lvert \overline{l}(\omega)f(i\omega) \right\lvert + \left\lvert \left( 1-e^{-i\eta_d\omega}f(i\omega) \right) w \right\lvert < 1
\end{equation}
where $w$ is a weighting used for determining the desired closed-loop system's bandwidth. A typical shape of $w^{-1}$ can be found in \cite{morari89}. In Section~\ref{sec:result}, we explain a process of determining $\eta_f$ in order to fulfil the robust performance criterion in Equation~\eqref{eq:robust_performance}.

The last but not least advantage of using the IMC structure is that we can optimally and robustly incorporate the input/output constraints \cite{turner04}. In the contact force control case, we mostly have output constraints such as the unilateral constraint of orthogonal forces to the contact surface and stiction/friction in the actuators. To this end, we have used a hinge-function block at the output of the nominal model in order to account for unilateral constraint for orthogonal contact surfaces. Moreover, to model the stiction/friction of the actuators, we use a deadzone block on the estimated disturbance $\tilde{d}$ signal. The latter can be interpreted as the case where the small difference between the nominal model's output force and the plant's output force is considered as the actuators' stiction/friction resisting force. In the next section, we study the coupling effect between the contact subsystem and other remaining DoF and introduce a generalized method to prevent this cross coupling.   


\begin{figure}[tbp]
    \includegraphics[width=\columnwidth]{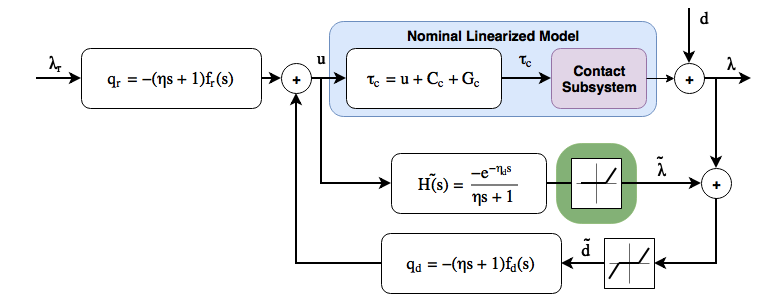}
    \caption{Overview of the contact force control structure. This controller uses feedback-linearization, the obtain a linearized plant. This then allows to use an IMC structure composed of the internal model $\tilde{H}_p(s)$, disturbance rejection controller $q_d(s)$, tracking controller $q_r(s)$, and output constraint blocks. Note that the hinge-function block (green) appears only in the control loop for orthogonal forces}.
    \label{fig:imc}
\end{figure}

\subsection{Task Space Control Decoupling}
Special care should be taken in mapping $\vtau_c = \vS_{c}^\top \vtau$. This is because the generalized forces $\vtau$ are not only manipulated by the contact subsystem's controller, but also by the other controllers acting on remaining DoF. This can produce undesirable couplings between controllers. In this section, we show how to systematically decouple the controllers for different DoF. To this end, we break down the legged robot model into three subsystems: contact subsystem, CoM subsystem, and non-contact subsystem as shown in Fig.~\ref{fig:control_structure}. Since the contact subsystem and its controller structure have already been introduced in the previous section we now briefly introduce the two other subsystems.

\paragraph{CoM subsystem} 
The CoM subsystem describes the equation of motion governing the CoM of the robot. As the Newton-Euler equations state the acceleration (angular and linear) of the CoM is proportional to the net external generalized forces (net torques and net forces). Therefore, the CoM subsystem is only affected by the external forces (such as contacts and disturbances) and is independent of the internal contact forces. Hence, the equation of motion is:
\begin{equation} \label{eq:com_subsystem}
\vM_{com}(\vq) \ddot{\vx}_{com} + \vC_{com}(\vq,\dot{\vq}) + \vG_{com}(\vq) = \vJ_{c, com}^\top \vlambda
\end{equation}
where $\vM_{com}$ consists of two 3-by-3 blocks on the main diagonal, where one block corresponds to the CoM angular momentum inertia and the other block to the total mass of robot multiplied by an identity matrix. The CoM subsystem has two distinct mechanisms for control. The first one is the contact forces of the stance feet in each of the phases of motion and the second one is the switching time between two consecutive motion phases and the contact point locations. Therefore, the CoM controller determines the desired contact forces and the switching times in order to track given trajectories. These desired forces are used as reference inputs to the contact subsystem controller and the switching times are used in the non-contact subsystem controller in order to establish/break contacts at given times and given locations.

\paragraph{Non-contact subsystem}
The non-contacting subsystem includes all the remaining DoF after excluding the contact subsystem and CoM subsystem. The task space Jacobian matrix for this subsystem can be defined in a way that either it picks the remaining generalized coordinates or it chooses any Cartesian task-dependent coordinates. In general, the equation of motion for this task does not have any special structure.
\begin{equation} \label{eq:noncontact_subsystem}
\vM_{nc}(\vq) \ddot{\vx}_{nc} + \vC_{nc}(\vq,\dot{\vq}) + \vG_{nc}(\vq) =\tau_{nc} +  \vJ_{c,nc}^\top \vlambda ,
\end{equation}
where $\ddot{\vx}_{nc}$ is the non-contact subsystem acceleration and $\tau_{nc} = \vS_{nc}^\top \vtau$ is the projected torque vector.

As Equations~\eqref{eq:contact_subsystem_1},~\eqref{eq:com_subsystem},~and~\eqref{eq:noncontact_subsystem} shows the contact and non-contact subsystems are the only ones that are directly affected by the generalized forces, $\tau$. Therefore, in order to avoid undesirable coupling between controllers of the contact and non-contact subsystems, the control input of subsystems, $\vtau_c$ and $\tau_{nc}$, should be decoupled.

In general, there are two strategies to deal with this coupling effect: using hierarchical tasks control approach or using complete decoupling approach; the latter we call: ``coequal tasks control''. The hierarchical tasks control sets low-priority task in the null space of the high priority task. Thus, the top level task can be designed independently whereas the lower level task needs to be informed of the higher task command. On the contrary, the coequal tasks control approach sets each of the two tasks in the null space of the other task. Hence, the tasks are decoupled into two completely independent tasks which allows for designing of each controller without the knowledge of the other task's control command. Theorem~\ref{th:coequal_task} introduces a generalized formulation for representing these two approaches.   

\begin{theorem} \label{th:coequal_task}
Assume that the contact subsystem and non-contact subsystem have full column rank input selection matrices $\vS_{c}$ and $\vS_{nc}$ respectively. The projected forces of subsystems, namely $\vtau_{c}$ and $\vtau_{nc}$, can be mapped by the following transformations to the generalized forces $\vtau$.
\begin{equation}
\vtau = \vW_{c} \vS_{c} \Big( \vS_{c}^\top \vW_1 \vS_{c} \Big)^{-1}\! \vtau_{c} + \vW_{nc} \vS_{nc} \Big( \vS_{nc}^\top \vW_2 \vS_{nc} \Big)^{-1}\! \vtau_{nc}
\end{equation}
with $\vW_{c}$ and $\vW_{nc}$ are defined as
\begin{align}
\label{eq:W1}
\vW_{c} &= \vW - \alpha_{nc} \vW \vS_{nc} \left( \vS_{nc}^\top \vW \vS_{nc} \right)^{-1} \vS_{nc}^\top \vW \\
\label{eq:W2}
\vW_{nc} &= \vW - \alpha_{c\ } \vW \vS_{c\ } \left( \vS_{c\ }^\top \vW \vS_{c\ } \right)^{-1} \vS_{c\ }^\top \vW 
\end{align}
where $\vW$ is an arbitrary full rank weighting matrix and $\alpha_{c}$, $\alpha_{nc} $ $\in \{0,1\}$. If both $\alpha$'s are $1$ then the tasks are coequal which means that one task control input does not directly affect the other subsystem's state evolution. If only one of them is $1$ then the tasks have a hierarchy (task with $\alpha_i=1$ has precedence). The case where both $\alpha$'s are $0$ does not have any specific structure. 
\end{theorem}

\begin{proof}
A simple way to prove this is to check the effect of one of the subsystem's generalized control input on the other subsystem's equation of motion. As an example, for the contact subsystem we have
\begin{equation*}
\vM_{c} \ddot{\vq}_{c} + \vC_{c} + \vG_{c} = \vS_{c}^\top \vtau +  \vJ_{c,c}^\top \vlambda
\end{equation*}
The generalized forces affect contact subsystem through the term $\vS_{c}^\top \vtau$. For this term we can write
\begin{align*}
\vS_{c}^\top \vtau &= \vS_{c}^\top \left( \vW_{c} \vS_{c} ( \vS_{c}^\top \vW_1 \vS_{c} )^{-1} \vtau_{c} + \vW_{nc} \vS_{nc} ( \vS_{nc}^\top \vW_2 \vS_{nc} )^{-1}\!\vtau_{nc} \right) \\ 
&=  \vtau_{c} + \vS_{c}^\top \vW_{nc} \vS_{nc} \left( \vS_{nc}^\top \vW_{nc} \vS_{nc} \right)^{-1} \!\vtau_{nc} 
\end{align*}
Using the definition of $\vW_{nc}$ in Equation~\eqref{eq:W2}, we obtain
\begin{equation}
\vS_{c}^\top \vtau = \vtau_{c} + (1-\alpha_{c}) \vS_{c}^\top \vW \vS_{nc} \left( \vS_{nc}^\top \vW \vS_{nc} \right)^{-1} \! \vtau_{nc} 
\end{equation}
For $\alpha_c=1$, contact subsystem will not be affected by $\vtau_{nc}$ and if $\alpha_c=0$ contact subsystem will be affected by the non-contact subsystem's projected force vector. Therefore, for $\alpha_{c}=\alpha_{nc}=1$ the two subsystems are independent (coequal), while if only one of the $\alpha$'s is $1$, the subsystem with $\alpha$ equal to $1$ will have precedence over the other subsystem.    
\end{proof}

You may notice that Theorem~\ref{th:coequal_task} requires the selection matrices to be full rank which is the case for fully actuated robots. However, for underactuated robots, this may not be the case. Under this circumstance, if we partition the selection matrix by unactuated and actuated parts, we can still use Theorem~\ref{th:coequal_task}. Fig.~\ref{fig:control_structure} shows relationships between subsystems using the coequal task space decomposition. 

\subsection{CoM Subsystem Controller}
In this work, we assume that the constraint satisfactory reference CoM motion is given by an external planner. In addition to reference pose and velocity of CoM \footnote{Note that the CoM orientation is not a measurable physical entity. Therefore, using its value in a feedback controller structure requires an estimation scheme through integrating CoM velocity which is prone to drift in the absence of direct measurements. There are two solutions to this problem (i) approximating the CoM orientation with the base orientation which we use in the case where the planner provides desired CoM acceleration. (ii) using the modeling approach introduced in \cite{farshidian17}. This approach normally designs the desired contact forces directly.}, the planner should either provide the desired acceleration of the CoM or the contact forces which realize the desired CoM motion. In order to deal with the discrepancies between the model and hardware, we use a PD controller on the CoM desired trajectories. This PD controller provides correction to the CoM acceleration which is later mapped to an equivalent correction of the desired contact forces. First, we assume the case where the planner provides the desired acceleration of the CoM. The corrected CoM acceleration can be calculated as: $\ddot{\vx}_{corr} = \ddot{\vx}_{com, ref} + K_d (\dot{\vx}_{com}-\dot{\vx}_{com, ref})\ + 
K_p (\vx_{com}-\vx_{com, ref})$,
%
%
where $\ddot{\vx}_{corr}$ is the corrected acceleration. Based on the CoM subsystem equation in \eqref{eq:com_subsystem}, the reference contact force should satisfy the following equation
\begin{equation*}
\vM_{com}(\vq) \ddot{\vx}_{corr} + \vC_{com}(\vq,\dot{\vq}) + \vG_{com}(\vq) = \vJ_{c, com}^{\top} \vlambda_r .
\end{equation*}
Depending on the number of the contact points and their configuration, the above linear matrix equation may have one, none, or multiple solutions. Furthermore, in order to produce valid desired contact forces, we need to respect the unilateral orthogonal contact force constraints and Coulomb friction limits. To this end, we have defined the following Quadratic Programming (QP) problem
\begin{equation}
\label{eq:program}
\begin{aligned}
& \underset{\vlambda_r}{\text{minimize}}
& & \Vert \vJ_{c, com}^{\top} \vlambda_r-\vF_r \Vert^2_2 \\
& \text{subject to}
& & \vlambda_Z \geq 0, \quad \Sigma\left[ \vlambda \right] \geq 0. 
\end{aligned}
\end{equation}
where $\vF_r = \vM_{com}(\vq) \ddot{\vx}_{corr} + \vC_{com}(\vq,\dot{\vq}) + \vG_{com}(\vq)$. The first set of inequality constraints are the unilateral constraint of the orthogonal forces and the second set represents the convex polyhedral approximation of the friction cones, $\Sigma\left[ \cdot \right]$. In our implementation, we have approximated the friction cone with an inscribed pyramid whose base is a twelve sided polygon. To solve this problem, we use the freely available QP solver QuadProg++ \cite{quadProg} which uses a dual method to solve the QP problem. In our C++ implementation, we are able to solve this QP problem in each step of the CoM control loop in less than $100 \mu s$. Finally, in the case where the planner directly provides the desired contact forces, we first map the desired contact forces to the net forces ($\vJ_{c, com}^{\top} \vlambda_{ref}$). Then, we add the PD controller correction force. Finally, we map the corrected net forces back to the contact forces while satisfy the contact constraints in Problem~\eqref{eq:program}.  

\subsection{Non-contact Subsystem Controller}
Except for possible user defined task, the main control objective for the non-contact subsystem is to establish contact in the next motion phase at a desired location. The controller on this subsystem should control the end-effectors (e.g. feet) to touch the contact surface at the given switching times. Here, we use inverse dynamics and a PD correction to track the reference inputs. We tune the controller to the lowest possible impedance that still achieves satisfactory tracking.

\section{Results} \label{sec:result}
In order to calculate the optimally robust $\eta_f$ which satisfies the robust performance inequality in Equation~\eqref{eq:robust_performance}, we need to determine an upper bound of the multiplicative uncertainty, $\overline{l}(\omega)$. We assume that the family $\Pi$ in Equation~\eqref{eq:contact_subsystem_unceratinity} includes first order deadtime systems. Since the delay in our system is fixed and only stems from known communication delay, we set $\eta_d = \tilde{\eta}_d = 0.003$. Then, we have approximated the uncertainties of DC gain, $k$, and the time constant, $\eta$, of the actuators based on the real data where we got the following bounds
\begin{equation*}
\begin{aligned}
&\left\vert k - \tilde{k} \right\vert \leq  \Delta k 
& & \tilde{k}=1.00, &\Delta k = 0.40 \\
&\left\vert \eta - \tilde{\eta} \right\vert \leq  \Delta \eta 
& & \tilde{\eta}=0.02, &\Delta \eta = 0.01
\end{aligned}
\end{equation*}
Based on this estimation for $\Delta k$ and $\Delta \eta$, we obtain the upper bound $\overline\l(\omega)$  in Equation~\eqref{eq:contact_subsystem_unceratinity} which is shown in Fig.~\ref{fig:robust_perfomrance} for different frequencies. Our weighting function, $w(s)$, has a unity weight up to 50~[$rad/s$] (which is the controller's bandwidth) then it rolls off with -20~[$dB/dec$]. Thus, $w(s)$ is equivalent to a first order filter with time constant 0.02. 

Since the reference tracking controller in IMC structure, $q_r(s)$, does not affect the closed-loop robustness, its lowpass filter is tuned to achieve smooth changes in the realized contact forces. The robust performance criterion (Equation~\eqref{eq:robust_performance}) for the disturbance rejection controller $q_d(s)$ is demonstrated in Fig.~\ref{fig:robust_perfomrance}. As it can be seen, for $\eta_f$ around $0.3$ this criteria is satisfied while decreasing $\eta_f$ causes the loss of guarantee on robust performance. For example for $\eta_f=0.01$, the red curve goes over $1$ in certain frequencies. Thus, we set $\eta_f=0.03$, in order to guarantee robust performance in disturbance rejection controller. In practice, it is better to add an extra fast pole to this filter (with around 10 times smaller response time) in order to filter out sensor noise. 

\begin{figure}[tbp]
    \includegraphics[width=\columnwidth]{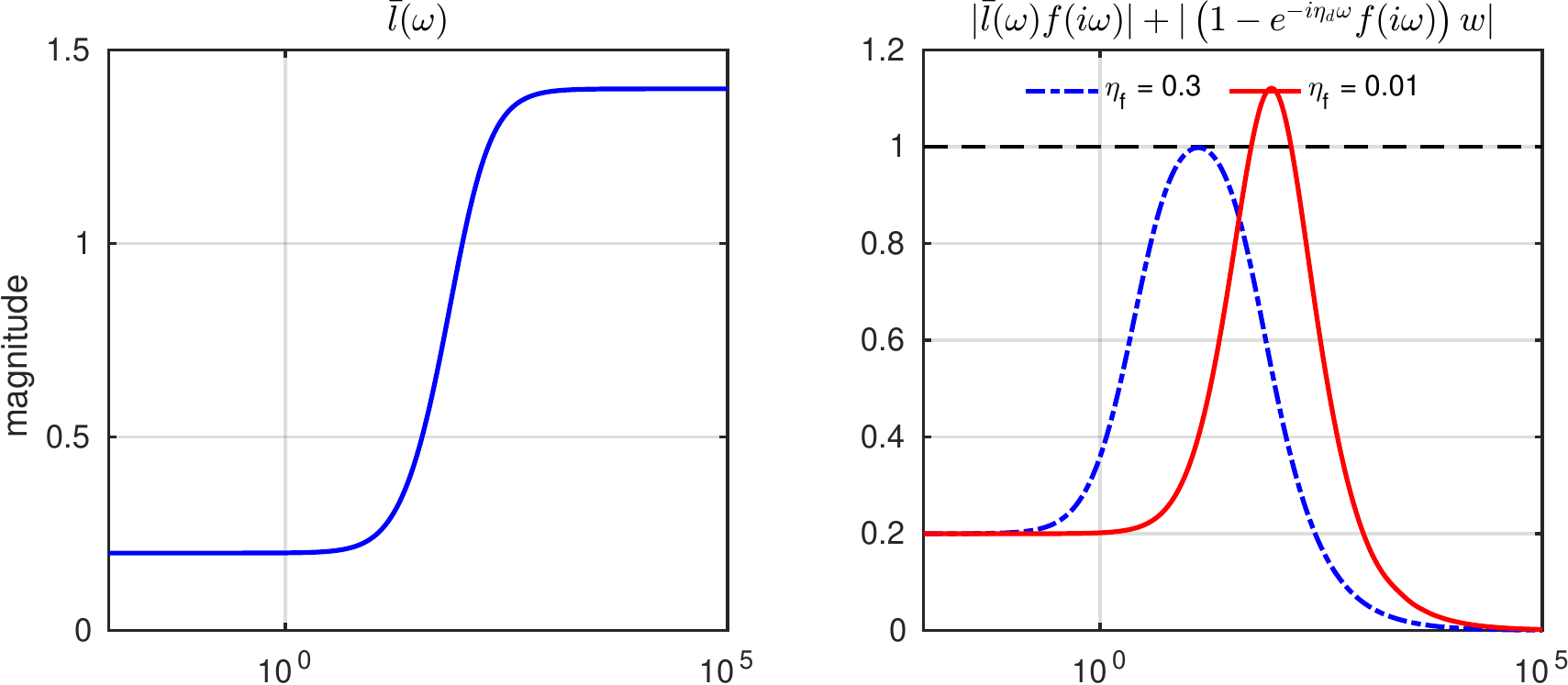}
    \caption{Bode plots of $\bar{l}(\omega)$ and left hand side in the Eq. \ref{eq:robust_performance}. As the weighting function ($w$) we use a first order filter with time constant 0.02. Hence, the controller robustly rejects disturbances up to the frequency of 50 [$rad/s$].}
    \label{fig:robust_perfomrance}
\end{figure}
    
We design a number of experiments to compare our approach of full body control against approaches that rely on inverse dynamics and PD control in the robot's joints (e.g. \cite{winkler17}) for tracking the reference trajectory. We start by showing the results obtained in simulation using SL simulation package \cite{Schaal2009}. In Fig.~\ref{fig:cmp_soft_sim_pos}, our approach is compared with the inverse dynamics based approach by showing the CoM tracking performance while executing the walking gait.

\begin{figure}[tbp]
    \includegraphics[width=\columnwidth]{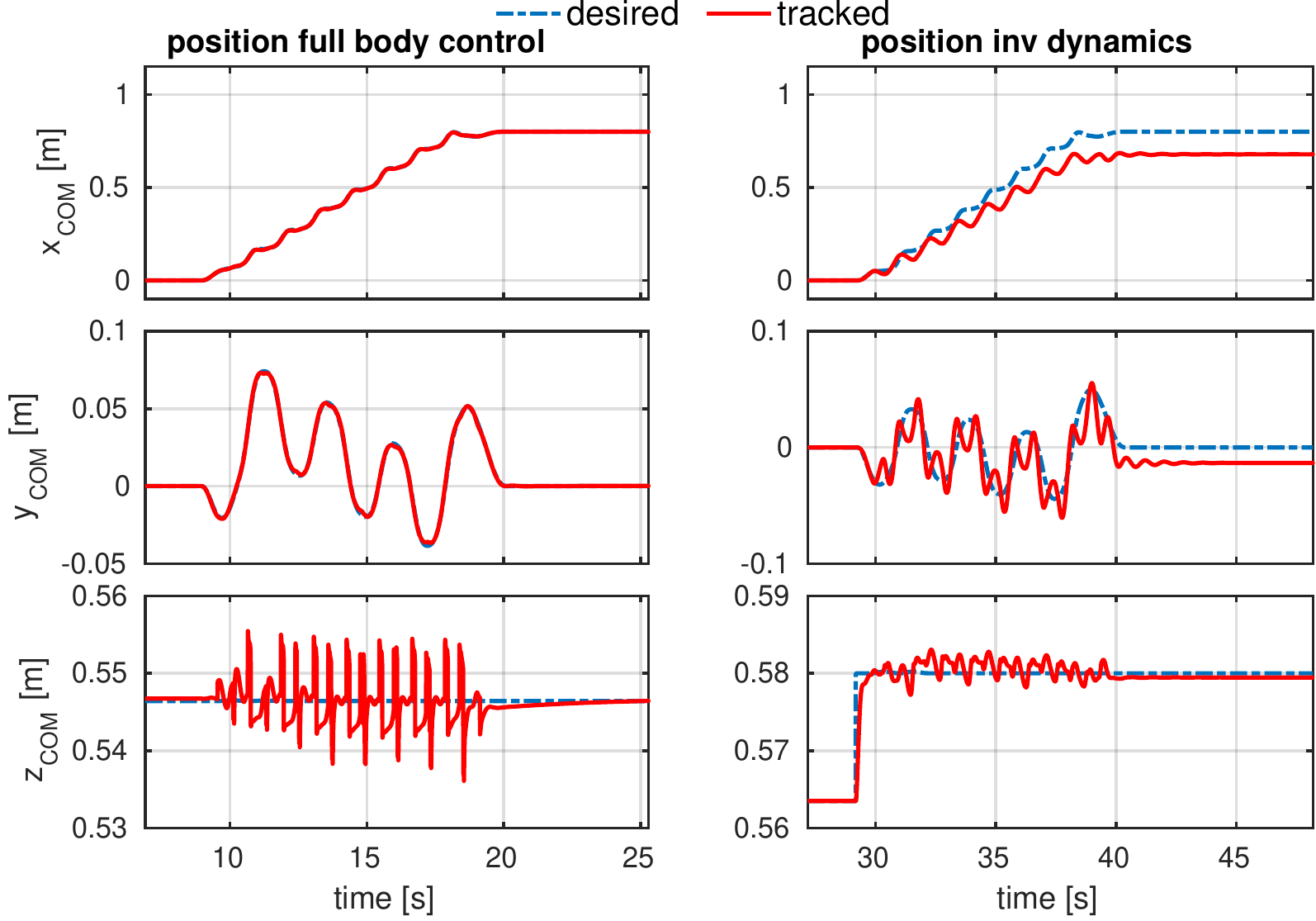}
    \caption{Simulation results for tracking performance in CoM position. Our full body control approach is shown on the left and the inverse dynamics based approach presented in \cite{winkler17RAL} is shown on the right}
    \label{fig:cmp_soft_sim_pos}
\end{figure}

The plots shown in Fig.~\ref{fig:cmp_soft_sim_pos} are obtained by simulating relatively low ground stiffness (spring constant of $1.0 \times 10^4 N/m$). While both approaches produce comparable tracking performance on a stiff ground ( $6.0 \times 10^5 N/m$), our approach brings significant improvement for walking on soft ground. The walking experiment was also conducted on a real robot. To simulate soft ground conditions, we wrapped HyQ's feet in foam and added soft rubbery ground for HyQ to walk on (see \href{https://youtu.be/bE2_-lpZU7o}{video} \footnote{The reader can find the videos of the
following simulation results online: \url{https://youtu.be/bE2_-lpZU7o} }). The comparison of tracking performances is shown in Fig.~\ref{fig:cmp_soft_hw_pos}. While both controllers achieve similar tracking performance in $y$ and $z$ directions, in $x$ direction the inverse dynamics based controller accumulates about 10~[$cm$] error in about 1~[$m$] of walking. On the other hand, our controller has less than 1~[$cm$] error in tracking of the $x$ direction.

To show that our whole-body control approach indeed actively controls end-effector contact forces without any auxiliary PD controllers in the joints, we conduct the seesaw and the plank experiments. In both experiments, the contact subsystem has no knowledge about changes in the robot's foothold positions. For the first experiment, we put a plank under the robot's front feet and lift it up while the robot's hind legs are still standing on the floor. The plank is displaced along the vertical axis; however, we also introduce lateral and longitudinal displacement in order to demonstrate the robustness of our approach (see \href{https://youtu.be/bE2_-lpZU7o}{video}). For the second experiment, HyQ is placed on a seesaw which we swing while the robot is maintaining the upright position of its base by actively controlling the feet's contact forces. Snapshots of the plank and seesaw experiments are shown in Fig.~\ref{fig:plank_exp_snap}~and~\ref{fig:seesaw_exp_snap}.

Disturbance rejection performance for the plank experiment is shown in Fig.~\ref{fig:plank_exp_pos} and the contact force tracking performance is shown in Fig.~\ref{fig:plank_exp_force}. In both figures, the light grey rectangles denote the time intervals during which the longitudinal and the lateral disturbances have been applied to the plank (see \href{https://youtu.be/bE2_-lpZU7o}{video}). We show the pitching angle $\beta$, the $z$ coordinate of the robot's base, the $z$ coordinate of the left front (LF) and the left hind (LH) feet as well as the contact forces for LF and LH feet. Comparing the z coordinate of the feet with the base pitching angle and z coordinate, we see that the effect of disturbance on the base is attenuated. This shows the performance of the contact force controller which robustly tracks the the desired forces under disturbances such as moving contact points. A similar results can be seen for the seesaw experiment in Fig.~\ref{fig:seesaw_exp_pos} and \ref{fig:seesaw_exp_force}. Like the plank experiment, the disturbance is considerably reduced. 

\begin{figure}[tbp]
    \includegraphics[width=\columnwidth]{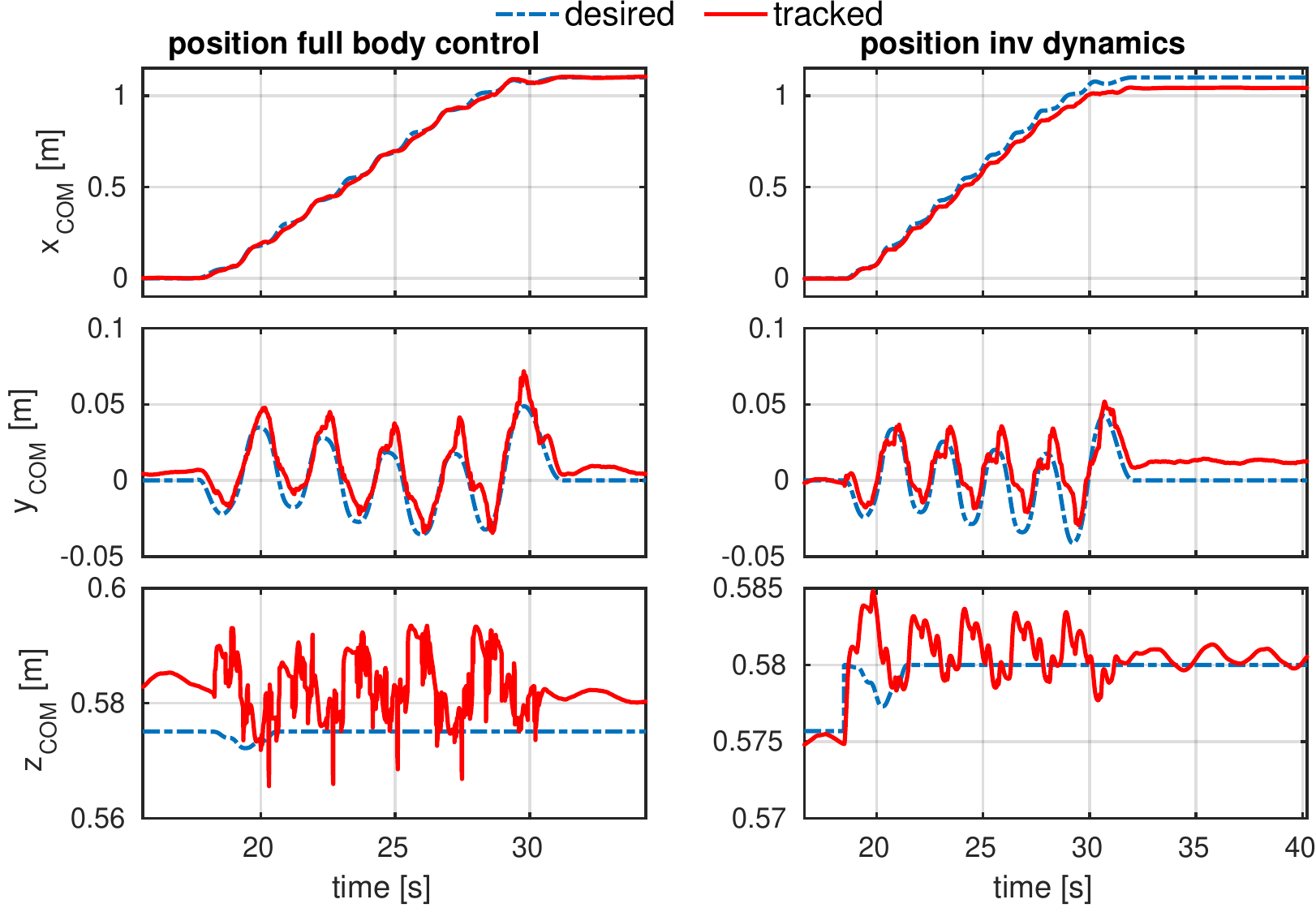}
    \caption{Hardware results for tracking performance in CoM position. Our full body control approach is shown on the left and the inverse dynamics based approach presented in \cite{winkler17RAL} is shown on the right.}
    \label{fig:cmp_soft_hw_pos}
\end{figure}

\section{Conclusion and Outlook}

In this paper, we presented a method for whole-body control of legged robots with the goal of tracking a planned trajectory for the CoM. The resulting full body controller has a feedback structure which increases the robustness in the presence of unmodeled actuator dynamics, changes in the ground stiffness and unobserved ground profile. The robust performance of our full body controller stems from the IMC control structure which we use to control the end-effector contact forces. The IMC structure allows us to incorporate model constraints specific for the contact subsystem, admits an easy and intuitive tuning process and provides theoretical guarantees on the robust performance of the resulting contact subsystem controller. 

Our full body controller relies on decomposing system dynamics into three subsystems. These three subsystems are the contact subsystem, the CoM subsystem and the non-contact subsystem. This decomposition allows us to design controllers separately for each system in a way such that there is no direct interference between different subsystems, which results in the better trajectory tracking for the CoM. We have demonstrated the effectiveness of our approach in simulation and on the real robot. Moreover, we have verified robustness of our approach by comparing its performance to the existing approach for legged robot motion control that relies mainly on inverse dynamics based controller.

In future, we would like to have this motion controller working together with our dynamic programming motion planner, OCS2,  \cite{farshidian16} which designs contact forces references and switching time based on the first principle of optimality and test our controller for high dynamic movements such as trotting. Moreover, we like to improve our contact force controller by incorporating contact force sensors in the control loop where we currently use state estimation algorithm.


\section*{Acknowledgement} \footnotesize{This research has been supported in part by a Max-Planck ETH Center for Learning Systems Ph.D. fellowship to Farbod Farshidian and a Swiss National Science Foundation Professorship Award to Jonas Buchli and the NCCR Robotics.}

\bibliographystyle{bibliography/IEEEtran} \bibliography{bibliography/references}

 \begin{figure*}[tbp]
    \includegraphics[width=\textwidth]{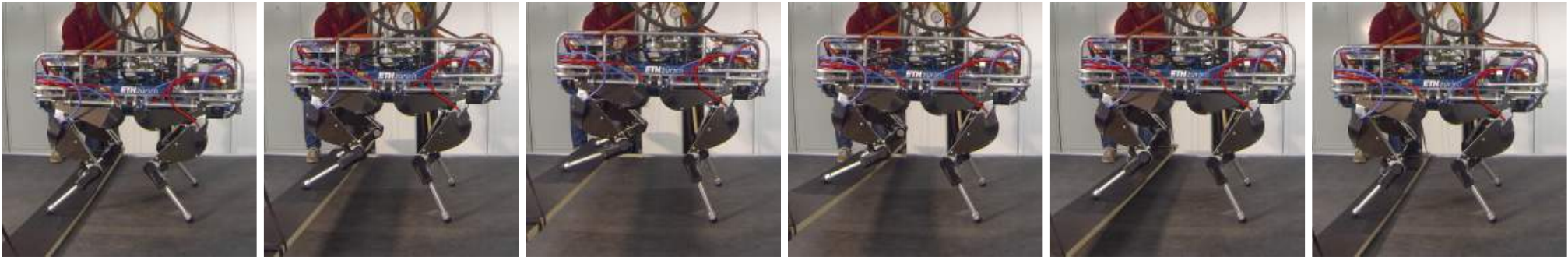}
    \caption{Plank experiment with HyQ. The plank is initially on the ground and then is lifted up. In order to keep its balance, the CoM controller modulates reference contact forces. The contact force controller has no knowledge of these changes. However, it robustly tracks the contact forces under the disturbances such as moving contact points.}
    \label{fig:plank_exp_snap}
\end{figure*}
 \begin{figure*}[tbph]
    \includegraphics[width=\textwidth]{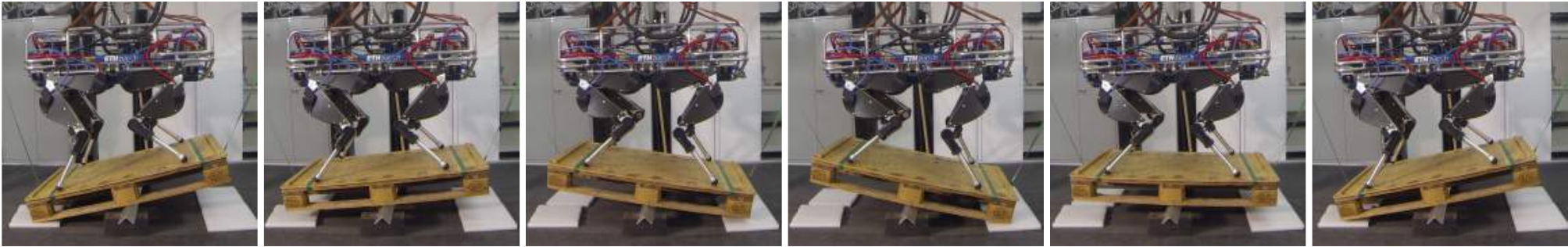}
    \caption{Seesaw experiment with HyQ. We swing the plank back and forth while the robot maintains the upright position.}
    \label{fig:seesaw_exp_snap}
\end{figure*}

\begin{figure}[tbp]
    \includegraphics[width=\columnwidth,height=5.5cm]{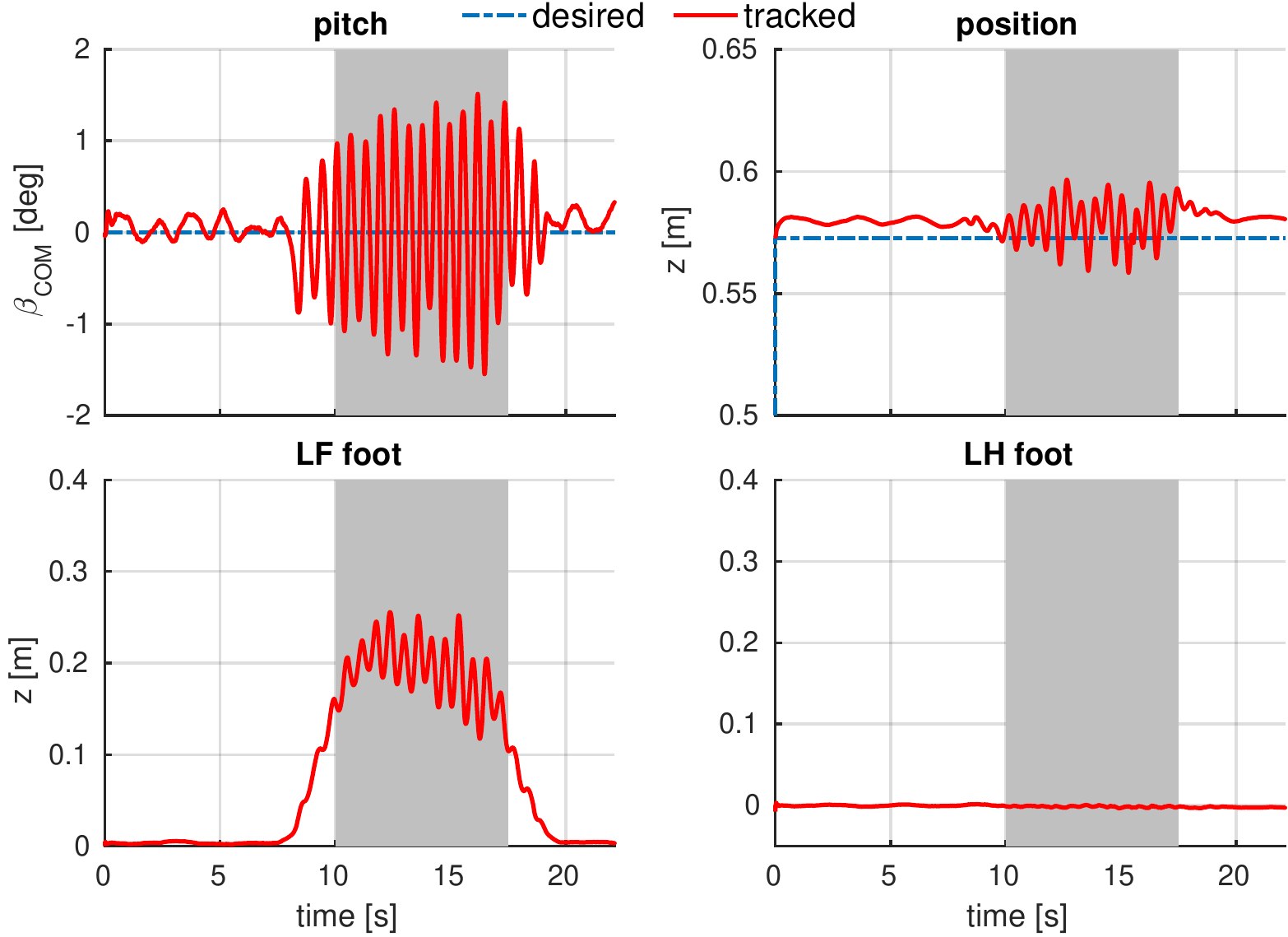}
    \caption{Position and pose tracking performance during the plank experiment. It can be seen that the robot curls its front legs and extends its hind legs to actively compensate for the disturbances. It manages to keep the pitching angle within 2 degrees deviation.}
    \label{fig:plank_exp_pos}
\end{figure}

\begin{figure}[tbp]
    \includegraphics[width=\columnwidth, height = 6cm]{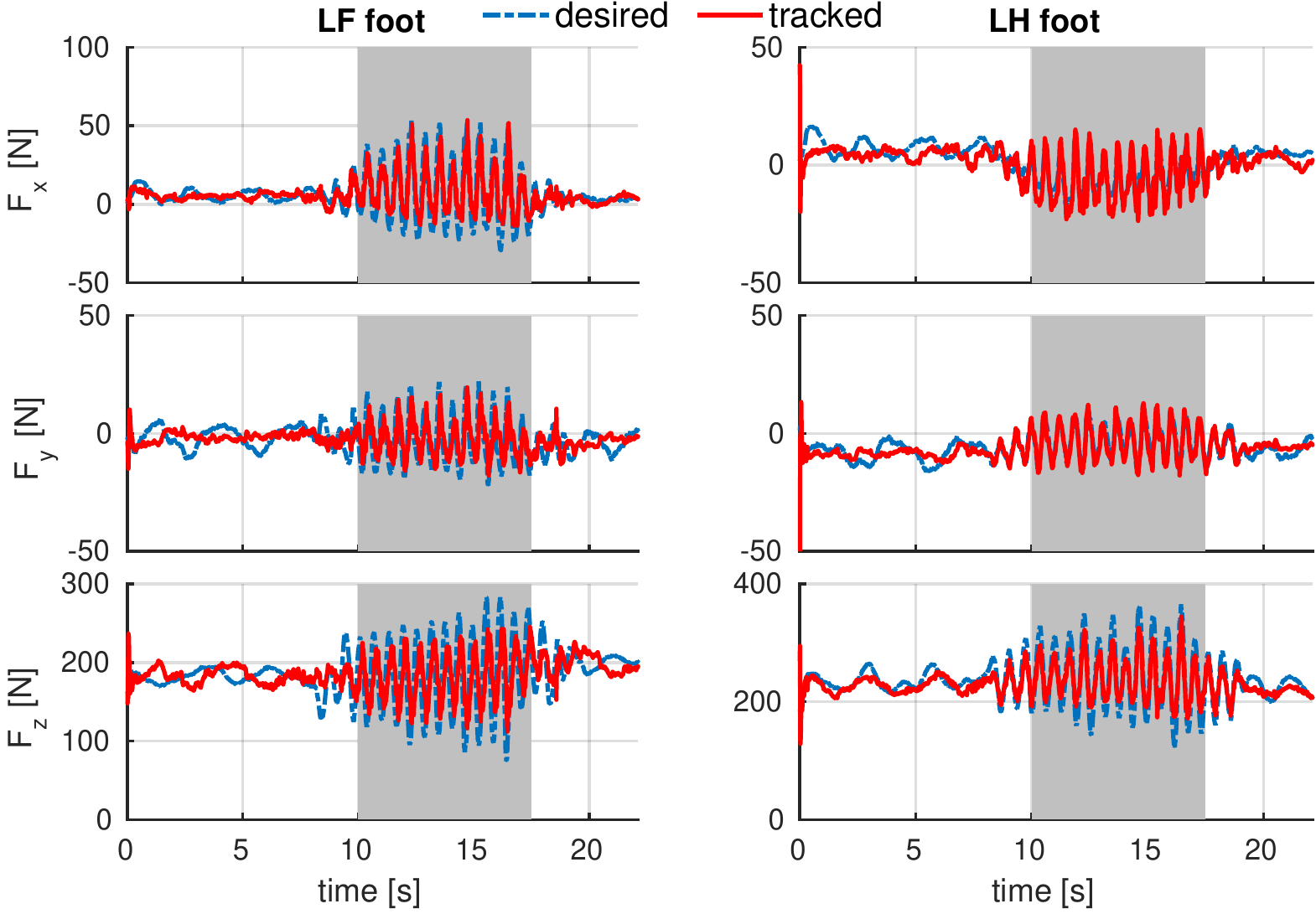}
    \caption{Tracking performance for the contact forces in $x$, $y$ and $z$ direction during the plank experiment. Besides changes in z direction, it can be seen that the controller also has to generate forces in $x$ and $y$ direction to compensate for the longitudinal and lateral displacement of the plank.}
    \label{fig:plank_exp_force}
\end{figure}

\begin{figure}[tbp]
    \includegraphics[width=\columnwidth,height=5.5cm]{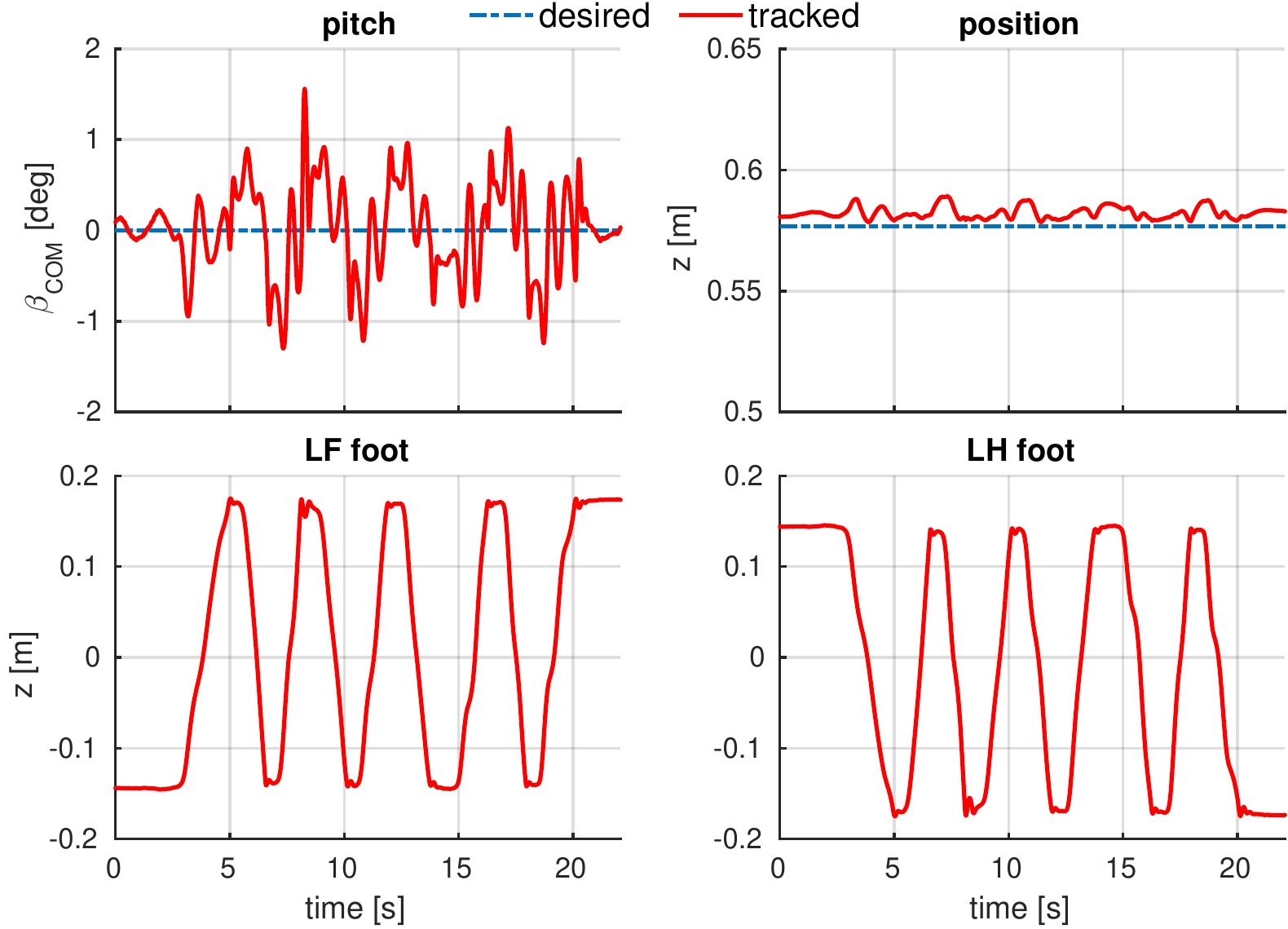}
    \caption{Position and pose tracking performance for the seesaw experiment. It can be seen that the robot curls front and hind legs in antiphase to keep the upright position of the base.}
    \label{fig:seesaw_exp_pos}
\end{figure}

\begin{figure}[tbp]
    \includegraphics[width=\columnwidth,height=6cm]{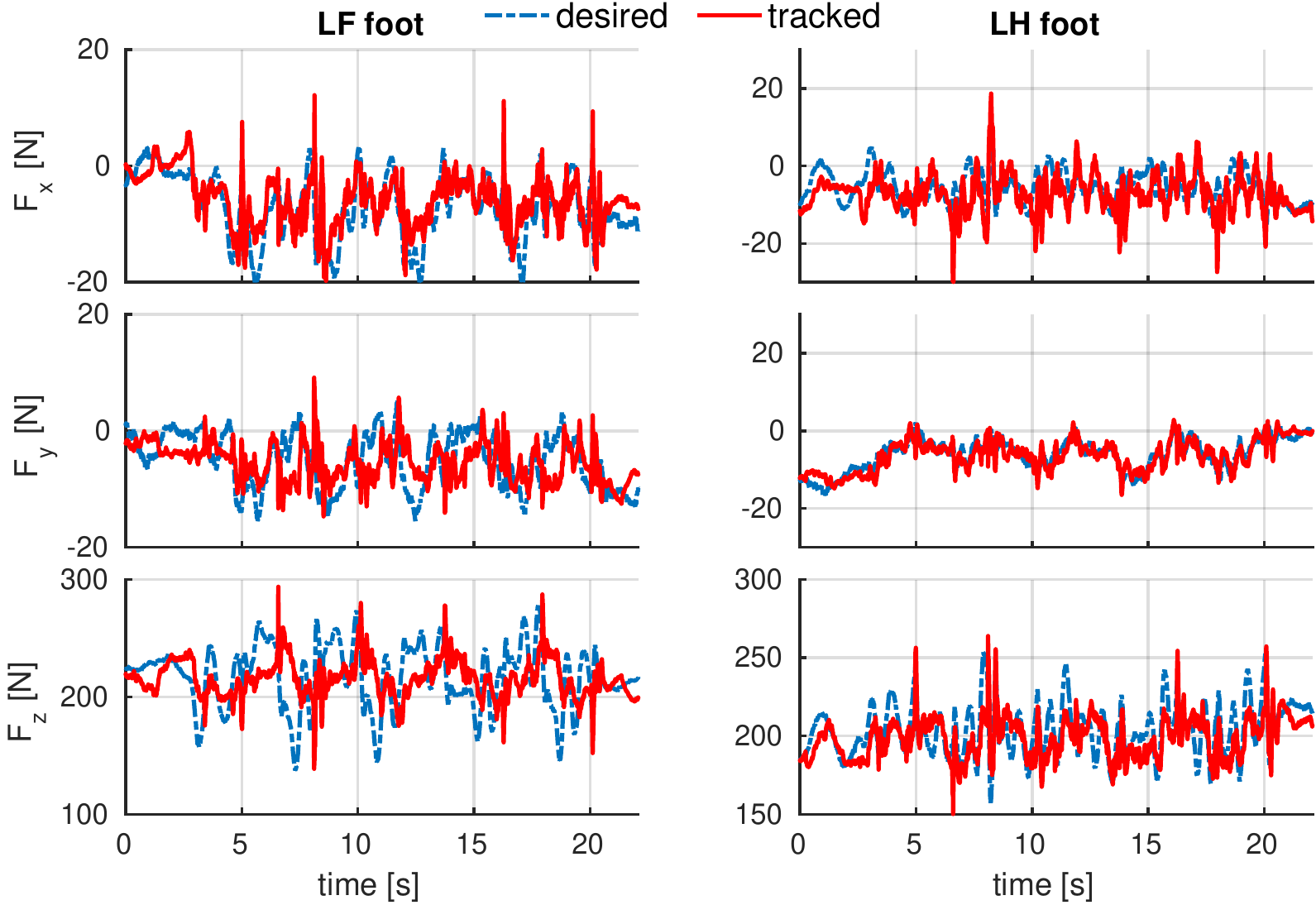}
    \caption{Tracking performance for the contact forces in $x$, $y$ and $z$ direction during the seesaw experiment. It can be seen that the changes in the $x$ and $y$ are not as salient as for the plank experiment, since no lateral or longitudinal displacement was introduced.}
    \label{fig:seesaw_exp_force}
\end{figure}

\end{document}

%% file: mathdef.tex
\newcommand{\vlambda}{\mbox{\boldmath $\lambda$}}

\newcommand{\vtau}{\mbox{\boldmath $\tau$}}

\newcommand{\vd}{\mathbf d}

\newcommand{\vq}{\mathbf q}

\newcommand{\vx}{\mathbf x}

\newcommand{\vC}{\mathbf C}

\newcommand{\vF}{\mathbf F}
\newcommand{\vG}{\mathbf G}

\newcommand{\vI}{\mathbf I}
\newcommand{\vJ}{\mathbf J}

\newcommand{\vM}{\mathbf M}

\newcommand{\vS}{\mathbf S}

\newcommand{\vW}{\mathbf W}

